\documentclass[11pt, letterpaper]{article}
\pdfoutput=1
\usepackage[authoryear]{natbib}
\usepackage{amsmath,amssymb,amsfonts,amsthm,bbm}
\usepackage{epic,eepic,epsfig,longtable}
\usepackage{multirow,verbatim}
\usepackage{array}

\usepackage[]{algorithm2e}

\usepackage{epsfig}
\usepackage{setspace}
\usepackage{url}
\usepackage{color}
\RequirePackage[colorlinks]{hyperref}
\hypersetup{
 colorlinks=false,
 citecolor=Violet,
 linkcolor=Red,
 urlcolor=Blue}

\newcommand{\SGD}{SGD\xspace}
\newcommand{\ASGD}{ASGD\xspace}
\newcommand{\AVG}{K-AVG\xspace}
\newcommand{\cifar}{\emph{CIFAR-10}\xspace}
\newcommand{\vgg}{\emph{vgg}\xspace}
\newcommand{\nin}{\emph{nin}\xspace}

\newcommand{\downpour}{\emph{Downpour}\xspace}
\newcommand{\eamsgd}{\emph{EAMSGD}\xspace}

\makeatletter
\def\singlespace{\def\baselinestretch{1}\@normalsize}

\makeatletter
\def\singlespace{\def\baselinestretch{1}\@normalsize}

\numberwithin{equation}{section}

\newcommand{\bfm}[1]{\ensuremath{\mathbf{#1}}}

     \def\EE{\mathbb{E}}

     \def\NN{\mathbb{N}}

     \def\RR{\mathbb{R}}

\def\bw{\bfm w}

\DeclareMathOperator{\argmin}{argmin}

\def\today{\ifcase\month\or
  January\or February\or March\or April\or May\or June\or
  July\or August\or September\or October\or November\or December\fi
  \space\number\day, \number\year}

\newdimen\biblioindent    \biblioindent=30pt

 at 8truept

\newcommand{\beq}{\begin{equation}}
  \newcommand{\eeq}{\end{equation}}
\newcommand{\beqn}{\begin{eqnarray}}
  \newcommand{\eeqn}{\end{eqnarray}}
\newcommand{\beqnn}{\begin{eqnarray*}}
  \newcommand{\eeqnn}{\end{eqnarray*}}

\allowdisplaybreaks
\usepackage{verbatim}
\renewcommand{\baselinestretch}{1.66}
\newcounter{CondCounter}

\numberwithin{equation}{section}
\theoremstyle{plain}
\newtheorem{theorem}{Theorem}[section]

\theoremstyle{definition}

\theoremstyle{definition}

\theoremstyle{definition}

\newtheorem{corollary}{Corollary}[section]

\theoremstyle{definition}
\newtheorem{assumption}{Assumption}

\theoremstyle{definition}

\theoremstyle{definition}

\usepackage{mathtools}

\begin{document}
\title{On the Convergence Properties of a $K$-step Averaging Stochastic Gradient Descent Algorithm for Nonconvex Optimization }
\author{
	Fan Zhou$^1$\thanks{Supported in part by NSF Grants DMS-1509739 and CCF-1523768}, 
	Guojing Cong$^2$, 
	\\ 
	$^1$ School of Mathematics, Georgia Institute of Technology \\
	$^2$ IBM Thomas J. Watson Research Center\\
	fzhou40@math.gatech.edu,
	gcong@us.ibm.com
}

\date{(\today)}
\maketitle

\begin{abstract}
	We adopt and analyze a synchronous K-step averaging stochastic gradient descent algorithm which
	we call \AVG  for solving large scale machine learning problems. We establish the convergence results 
	of \AVG for nonconvex objectives. Our analysis of \AVG applies to
	many existing variants of synchronous SGD.  We explain why 
	the K-step delay is necessary and leads to better performance than traditional parallel stochastic gradient descent which is equivalent to \AVG with $K=1$.
	We also show that K-AVG scales better with the number of learners than
	asynchronous stochastic gradient descent (ASGD).
	Another advantage of \AVG over \ASGD  is that it allows larger
	stepsizes and facilitates faster convergence. 
	On a cluster of $128$ GPUs, \AVG is faster than \ASGD implementations
	and achieves better accuracies and faster convergence for training
	with the \cifar dataset.
\end{abstract}

\section{Introduction}
\label{introduction}
Parallel and distributed processing have been adopted for stochastic
optimization to solve large-scale
machine learning problems. Efficient parallelization
is critical to accelerating long running deep-learning
applications. Derived from stochastic gradient descent
(SGD), parallel solvers such as synchronous \SGD (e.g.,  \cite{zinkevich2010parallelized} \cite{dekel2012optimal}) and Asynchronous \SGD
(ASGD) (e.g., see \cite{dean2012large,recht2011hogwild}), have been proposed.

Beginning with the seminal paper \cite{robbins1951stochastic},
the convergence properties of \SGD and its variants have been extensively studied for the past 50 years (e.g. see \cite{robbins1971convergence,bottou1998online,nemirovski2009robust,shamir2013stochastic,ghadimi2013stochastic}).  
The asymptotic optimal convergence rate of \SGD was proved by \cite{chung1954stochastic} and \cite{sacks1958asymptotic} to be $O(1/N)$ with twice continuously differentiable and strongly 
convex objectives. $N$ is the
number of samples processed. The iteration complexity is
$O(1/\sqrt{N})$ for general convex (see \cite{nemirovski2009robust}) and nonconvex (see \cite{ghadimi2013stochastic}) problems.
Regarding parallel variants of \SGD,
\cite{dekel2012optimal} extend these results to the setting of synchronous \SGD
with $P$ learners and show that it has convergence rate of
$O(1/\sqrt{NP})$ for non-convex objectives, with $N$ being the
number of samples processed by each learner. Hogwild! is a lockfree implementation of \ASGD, and \cite{recht2011hogwild} prove its convergence for strongly convex problems with  
theoretical linear speedup over SGD. \downpour is another
\ASGD implementation with resilience against machine
failures \cite{dean2012large}. \cite{lian2015asynchronous} show that as long as the
gradient staleness is bounded by the number of learners, \ASGD
converges for nonconvex problems. Due to its asynchronous nature that reduces communication cost, ASGD receives much attention in many recent studies. 

Although \ASGD has the same
asymptotic convergence rate as \SGD when the staleness of gradient
update is bounded, the learning rate assumed for proving \ASGD convergence are usually too small for
practical purposes. It is also difficult for an \ASGD
implementation to control the staleness in
gradient updates as it is influenced by the relative processing speed
of learners and their positions in the communication network. Furthermore,
the parameter server presents performance challenges on platforms with many
GPUs. On such platforms, a single parameter server
oftentimes does not serve the aggregation requests fast enough.  A
sharded server alleviates the aggregation bottleneck but introduces
inconsistencies for parameters distributed on multiple
shards. Communication between the parameter server (typically on CPUs) 
and the learners (on GPUs) is likely to remain a bottleneck in future
systems. 

We adopt a distributed, bulk-synchronous SGD algorithm that allows delayed gradient aggregation to effectively minimize the communication
overhead.  We call this algorithm K-step average SGD (\AVG). Instead of using a parameter
server,  the learners in K-AVG compute the average of their copies of parameters at regular intervals through global reduction. Rather than relying on asynchrony that
reduces communication overhead but has adverse impact on practical
convergence,  the communication interval $K$ is a
parameter in K-AVG. 
The communication time is amortized among the
data samples processed within each interval. On current and emerging computer platforms that
support high bandwidth direct communication among GPUs (e.g.,
GPU-direct), global reduction does not involve CPUs and
avoids multiple costly copies through the
software layers. Similar averaging approaches have been proposed in the literature, see \cite{hazan2014beyond,johnson2013accelerating,smith2016cocoa,zhang2016parallel,loshchilov2016sgdr,chen2016revisiting,wang2017memory}. However, their convergence behavior is not well understood analytically for nonconvex objectives, and it is unclear how they compare with \ASGD approaches. This is the part where our major contribution goes to.

We study the convergence behavior of \AVG and the impact of the number
of processors $P$ on convergence. As the
abundance of data is critical to the success of most machine learning
tasks, training employs increasingly more learners.  
We show that K-AVG scales better than \ASGD with $P$ and \AVG
allows larger stepsizes than \ASGD for the same $P$.
We also analyze the impact of $K$ on convergence.  Since with $K=1$
\AVG becomes hard-sync SGD, and with larger $K$ \AVG can stimulate
averaging after an epoch or many epochs,  our analysis can be applied
to many existing variants of synchronous SGD.  Finding the optimal
length of delay $K_{opt}$ for convergence is of high importance to
practitioners. 
Contrary to popular belief, $K_{opt}$ is oftentimes not $1$ and can be
very large for many applications.  Thus \AVG is a good fit for
large-scale distributed training as communication may not need to be
very frequent for optimal convergence.  Our analysis of convergence of
\AVG provides
guidelines practitioners to explicitly balance the decrease of communication time and the
increase of iterations through an appropriately chosen $K$.

Using an image recognition benchmark, we demonstrate the
nice convergence properties of \AVG in comparison to two popular \ASGD implementations:
\downpour~\cite{dean2012large} and \eamsgd~\cite{zhang2015deep}. In \eamsgd,
global gradient aggregation among learners simulates an elastic force
that links the parameters they compute with a center variable stored
by the parameter server. In both \downpour and \eamsgd, updates to the
central parameter server can also have a K-step delay. On our target
platform, when $K$ is small, \AVG significantly reduces the
communication time in comparison to \downpour and \eamsgd while
achieving similar training and test accuracies.  The training time
reduction is up to 50\%. When $K$
is large, \AVG achieves much better training and test accuracies than
\downpour and \eamsgd after the same amount of data
samples are processed.  For example, with 128 GPUs, 
\AVG is up to about $7$ and $2$-$6$ times faster than \downpour and \eamsgd respectively,
and achieves significantly better accuracy.

This rest of the paper is organized as follows: In section 2, we introduce the standard assumptions 
in optimization theory needed to analyze SGD methods and frequently used notations throughout the paper;
In section 3, we formally introduce the K-AVG algorithm, and prove its standard convergence results with fixed and diminishing stepsize.
Based on the convergence result, we analyze the scalability of \AVG and investigate the optimal choice of $K$; In section 4, we present our
experimental results to validate our analysis.

\section{Preliminaries and notations}
In this section, we introduce some standard assumptions used in the analysis of convex and non-convex optimization algorithms and key notations frequently used throughout this paper. We use $\|\cdot\|_2$ to denote the $\ell_2$ norm of a vector in $\RR^d$; $\langle \cdot \rangle$ to denote the general inner product in $\RR^d$. For the key parameters we use:
\begin{itemize}
	\item $P$ denotes the number of processors;
	\item $K$ denotes the length of the delay;
	\item $B_n$, $\bar{B}$, or $B$ denotes the size of mini-batch for $n$-th update;
	\item $\gamma_n$, $\bar{\gamma}$, or $\gamma$ denotes the step size for $n$-th update;
	\item $\xi^j_{k,s}$ denotes the i.i.d. realizations of a random variable $\xi$ generated by the algorithm on different processors and in different iterations, especially,
	$j=1,...,N$, $k=1,...,K$, and $s=1,...,B$. 
\end{itemize}

We study the following optimization problem:
\begin{equation}
\min\limits_{\bw \in \mathcal{X}} F(\bw)
\end{equation}
where objective function $F:\RR^m \rightarrow \RR$ is continuously differentiable but not necessarily convex over $\mathcal{X}$, and $\mathcal{X}\subset \RR^m$ is a nonempty open subset. Since our analysis is in a very general setting, $F$ can be understood as 
both the expected risk $F(\bw) = \EE f(\bw;\xi) $ or the empirical risk $F(\bw) = n^{-1}\sum_{i=1}^n f_i(\bw)$.  
As our approach for analysis is built upon smooth objectives, we introduce the following assumptions which are standard and fundamental.

\begin{assumption}
	\label{Lipschitz}
	The objective function $F:\RR^d \rightarrow \RR$ is continuously differentiable and the gradient function of $F$ is Lipschitz continuous with Lipschitz constant $L>0$, i.e.
	$$
	\big\| \nabla F(\bw) - \nabla F(\widetilde{\bw})\big\|_2 \leq L\big\| \bw - \widetilde{\bw}\big\|_2
	$$
	for all $\bw$, $\widetilde{\bw}\in \RR^d$. 
\end{assumption}
This assumption is essential to convergence analysis of our algorithm as well as most gradient based ones. Under such an assumption, the gradient of $F$ serves as a good indicator for how far
to move to decrease $F$.

\begin{assumption}
	\label{lowerbound}
	The sequence of iterates $\{\bw_j\}$ is contained in an open set over which $F$ is bounded below by a scalar $F^*$.
\end{assumption}
Assumption \ref{lowerbound} requires that objective function to be bounded from below, which guarantees the problem we study is well defined.
\begin{assumption}
	\label{unbias}
	For any fixed parameter $\bw$, the stochastic gradient $\nabla F(\bw; \xi)$ is an unbiased estimator of the true gradient 
	corresponding to the parameter $\bw$, namely, 
	$$
	\EE_{\xi} \nabla F(\bw; \xi) = \nabla F(\bw).
	$$
\end{assumption}
One should notice that the unbiasedness assumption here can be replaced by a weaker version which is called the First Limit Assumption see \cite{bottou2016optimization} that can still be applied to 
our analysis. For simplicity, we just assume that the stochastic gradient is an unbiased estimator of the true one.

\begin{assumption}
	\label{variance}
	There exist scalars $M \geq 0$ such that,
	$$
	\EE_{\xi} \big\| \nabla F(\bw;\xi)\big\|_2^2 - \big\|\EE_{\xi} \nabla F(\bw;\xi) \big\|_2^2 \leq M.
	$$
\end{assumption}
Assumption \ref{variance} characterizes the variance (second order moments) of the stochastic gradients. 

Since all results in this paper are based on the four assumptions above, we present them with no further mention in the following literature. 

\section{Main results}

We present the distributed K-AVG algorithm as follows:

\begin{algorithm}[H]
	\label{algorithm1}
	initialize $\widetilde{\bw}_1$\;
	\For{$n=1,...,N$}{
		Processor $P_j$, $j=1,\dots,P$ do concurrently:\\
		set $\bw_n^j=\widetilde{\bw}_n$ \;
		\For{$k=1,...,K$}{
			
			randomly sample a mini-batch of size $B_n$ and update:
			$$
			\bw_{n+k}^j = \bw_{n+k-1}^j - \frac{\gamma_n}{B_n} \sum\limits_{s=1}^{B_n} \nabla F(\bw_{n+k-1}^j;\xi_{k,s}^j)
			$$
		}
		Synchronize $\widetilde{\bw}_{n+1} = \frac{1}{P}\sum\limits_{j=1}^P \bw_{n+K}^j$\;
	}
	\caption{K-step average stochastic gradient descent algorithm}
\end{algorithm}\

Note that traditional parallel SGD algorithm is equivalent to \AVG with $K=1$, as synchronization is required after each local update. However, \AVG relaxes this requirement and allows for $K$ individual updates before
synchronization. Thus, K-AVG is a more general synchronous algorithm that contains parallel SGD.
Surprisingly, as we show both analytically (section \ref{subsection:K}) and experimentally (section \ref{experiment:K}), more frequent synchronization does not always result in faster convergence
for nonconvex objectives.

\subsection{Convergence of \AVG}

In the following theorem, we prove an upper bound on the expected average squared gradient norms, which serve as a metric to measure the convergence rate 
for nonconvex objectives.
\begin{theorem}(Nonconvex objective, fixed stepsize, and fixed batch size)
	\label{theorem:fixed stepsize}
	Suppose that Algorithm \ref{algorithm1} is run with a fixed stepsize $\gamma_n=\bar{\gamma}$, a fixed batch size $B_{n} = \bar{B}$ satisfying 
	\begin{equation}
	\label{cond:stepsize}
	1 \geq  \frac{L^2\bar{\gamma}^2(K+1)(K-2)}{2} + L\bar{\gamma} K,~~~and~~~1-\delta\geq L^2\bar{\gamma}^2
	\end{equation}
	with some constant $0<\delta<1$.
	Then the expected average squared gradient norms of $F$ satisfy the following bounds for all $N\in \NN$:
	\begin{equation}
	\label{fixed2}	
	\frac{1}{N}\EE \sum\limits_{j=1}^N \big\| \nabla F(\widetilde{\bw}_j)\big\|_2^2 \leq \Big[\frac{2(F(\widetilde{\bw}_1)-F^*)}{N(K-1+\delta)\bar{\gamma}} + \frac{LK\bar{\gamma} M}{\bar{B}(K-1+\delta)} \Big(\frac{K}{P} + \frac{L(2K-1)(K-1)\bar{\gamma}}{6}\Big) \Big]  ,	
	\end{equation}
\end{theorem}
The proof of Theorem \ref{theorem:fixed stepsize} can be found in section \ref{proof:theorem fixed}.
An immediate observation from (\ref{fixed2}) is that the expected average squared gradient norms of $F$ converges to some nonzero constant as $N\rightarrow \infty$. 
The first term $\frac{ 2(F(\widetilde{\bw}_1)-F^*)}{N(K-1+\delta)\bar{\gamma}}$ reflects the distance from initial weight to the solution. It eventually goes to zero as the number of iterations
goes to infinity. The second term $\frac{LK\bar{\gamma} M}{\bar{B}(K-1+\delta)} \Big(\frac{K}{P} + \frac{L(2K-1)(K-1)\bar{\gamma}}{6}\Big)$ is not affected by the iteration number. Compared with sequential SGD (see section 4.3 in \cite{bottou2016optimization}), this term is scaled by the batch size $1/\bar{B}$, and $1/P$ or $\frac{L(2K-1)(K-1)\bar{\gamma}}{6}$, which means larger batch size and smaller stepsize, more learners or more frequent averaging tend to reduce this term.
Mini-batch method as a variance reduction technique explains the appearance of $\bar{B}$. Parallelization of this algorithm contributes to the scaling factor $1/P$. However, when $P$ is large enough to make $\frac{L(2K-1)(K-1)\bar{\gamma}}{6}$ dominates $K/P$, the effect of parallelization is not ideal as one may expect. 

The rates of convergence (also referred as iteration complexity) after $N$ step updates are established in the following corollary which originates from Theorem \ref{theorem:fixed stepsize}.

\begin{corollary}
	\label{corollary}
	Under the condition of Theorem \ref{theorem:fixed stepsize},
	take
	\begin{equation}
	\label{corollary:conditionP}
	\bar{\gamma} = \sqrt{\frac{(F(\widetilde{\bw}_1)-F^*)\bar{B}P}{LMK^2N}}
	\end{equation}
	Then for any 
	$$
	N \geq \frac{(F(\widetilde{\bw}_1)-F^*)L\bar{B}P}{M}\Big( \frac{P^2}{K^2}\bigvee 1 \Big),
	$$
	the following bound is achieved
	\begin{equation}
	\label{corollary:bound1}
	\EE \frac{1}{N}\sum\limits_{n=1}^N  \big\| \nabla F(\widetilde{\bw}_n)\big\|_2^2 \leq \Big( \frac{4K}{K-1+\delta}\Big)\sqrt{\frac{(F(\widetilde{\bw}_1)-F^*)LM}{\bar{B}P}}*\frac{1}{\sqrt{N}}.
	\end{equation}
\end{corollary}

The proof of Corollary \ref{corollary} can be found in section \ref{proof:corollary}.
Condition (\ref{corollary:conditionP}) and bound (\ref{corollary:bound1}) imply when the number of updates $N$ is large enough, K-AVG eventually achieves a similar rate of convergence as classical SGD method for nonconvex objectives. Indeed, the rate of 
convergence of classical SGD methods is $N^{-1/2}$ after $N$ samples processed. Note that $N$ updates in K-AVG means that $N*K*B*P$ samples have been processed. Taking a closer look at bound in (\ref{corollary:bound1}), the right hand side is of the order $O((N*B*P)^{-1/2})$. K-AVG loses a factor of $1/\sqrt{K}$ as a result of communication saving. However, this doesn't mean that 
the smaller $K$ the better. Since there is an extra multiplicative factor $\frac{4K}{K-1+\delta}$ which is monotone decreasing with respect to $K$. We will have a more detailed discussion on the choice of $K$ in section \ref{subsection:K}.

Deploying diminishing stepsizes and/or dynamic batch sizes makes the expected average squared gradient norms converge to zero for non-convex optimization. In the following theorem, we establish the convergence result under such conditions. 

\begin{theorem}(Nonconvex objective, diminishing step size, and growing batch size)
	\label{theorem:dynamicsize}
	Suppose that Algorithm \ref{algorithm1} is run with diminishing step size $\gamma_j$, and growing batch size $B_j$ satisfying 
	\begin{equation}
	1 \geq  \frac{L^2\bar{\gamma}_j^2(K+1)(K-2)}{2} + L\bar{\gamma}_j K,~~~~and~1-\delta \geq L^2\gamma_j^2
	\end{equation}
	with some constant $0<\delta<1$.
	Then the weighted average squared gradient norms satisfies
	\begin{equation}
	\begin{aligned}
	\label{change1}
	&\EE  \sum\limits_{j=1}^N \frac{\gamma_j}{\sum_{j=1}^N \gamma_j}\big\| \nabla F(\widetilde{\bw}_j)\big\|^2_2 \\
	& \leq \frac{ 2(F(\widetilde{\bw}_1)-F^*)}{(K-1+\delta)\sum_{j=1}^N \gamma_j} + \sum\limits_{j=1}^N\frac{LK\gamma_j^2M}{B_j(K-1+\delta)\sum_{j=1}^N \gamma_j}  \Big(\frac{K}{P} + \frac{L(2K-1)(K-1)\gamma_j}{6}\Big) 
	\end{aligned}
	\end{equation}
	Especially, if
	\begin{equation}
	\label{batchsize}
	\lim\limits_{N\rightarrow \infty} \sum\limits_{j=1}^{N} \gamma_j = \infty,~~ \lim\limits_{N\rightarrow \infty} \sum\limits_{j=1}^N\frac{K\gamma_j^2}{PB_j} < \infty,~~\lim\limits_{N\rightarrow \infty} \sum\limits_{j=1}^N \gamma_j^3 < \infty,
	\end{equation}
	Then
	$$
	\EE  \sum\limits_{j=1}^N \frac{\gamma_j}{\sum_{j=1}^N \gamma_j}\big\| \nabla F(\widetilde{\bw}_j)\big\|^2_2 \rightarrow 0,~as~N\rightarrow \infty.
	$$
\end{theorem}
The proof of Theorem \ref{theorem:dynamicsize} can be found in section \ref{proof:theorem dynamic}. As we can see, by adopting a diminishing sequence of stepsizes instead of a fixed one,
the expected average squared gradient norms of \AVG converges to $0$ instead of a nonzero constant.

\subsection{K-AVG allows for larger stepsize than \ASGD}
Compared with the classical stepsize schedule for both sequential SGD (proposed by \cite{robbins1951stochastic}) and ASGD:
$$
\sum\limits_{j=1}^{\infty} \gamma_j = \infty,~\sum\limits_{j=1}^{\infty} \gamma_j^2 < \infty;
$$
the stepsize schedule proposed in (\ref{batchsize}) turns out to allow larger choices of $\gamma_j$. 
On one hand, $\sum_{j=1}^{\infty} \gamma^3_j < \infty$ itself is a much more relaxed constrain compared with $\sum_{j=1}^{\infty} \gamma^2_j < \infty$. On the other hand, as a byproduct of parallelization, when $P$ is large,
$\sum_{j=1}^{\infty} \gamma^2_j/B_jP$ also allows larger choice of $\gamma_j$. Intuitively, averaging acts as a variance reduction and leads to relaxation of larger stepsize constrain. In our experiments (section \ref{experiment:with ASGD}), larger stepsizes work well in K-AVG but can result in divergence in popular \ASGD implementations.

\subsection{Scalability comparison of \AVG against \ASGD}
\label{subsection:P}
We analyze the bound on expected average squared gradient norms in (\ref{fixed2}) to show \AVG algorithm scales better with $P$ than \ASGD.
We first establish the following theorem on the scalability of K-AVG.
\begin{theorem}
	\label{scalability}
	Under the condition of Theorem \ref{theorem:fixed stepsize},
	K-AVG scales better than ASGD.
\end{theorem}
\begin{proof}
	In \ASGD, one key parameter is the maximum staleness, generally assumed to be upper bounded by the number of processors, i.e. $P$ in \AVG. 
	With fixed stepsize, the expected average gradient norms is (see also \cite{lian2015asynchronous}, Theorem 3) is
	\begin{equation}
	\label{asynchronous}
	\EE \frac{1}{N}\sum\limits_{n=1}^N  \big\| \nabla F(\widetilde{\bw}_n)\big\|_2^2 \leq \Big[ \frac{ C_0(F(\widetilde{\bw}_1)-F^*)}{N\bar{\gamma}} + \frac{C_1L^2\bar{\gamma}^2 M^2P}{2\bar{B}} \Big].
	\end{equation}
	where $C_0$ and $C_1$ are constants independent of $P$. 
	
	Compared with the bound in (\ref{fixed2}), 
	$$
	\frac{1}{N}\EE \sum\limits_{j=1}^N \big\| \nabla F(\widetilde{\bw}_j)\big\|_2^2 \leq \Big[\frac{2(F(\widetilde{\bw}_1)-F^*)}{N(K-1+\delta)\bar{\gamma}} + \frac{LK\bar{\gamma} M}{\bar{B}(K-1+\delta)} \Big(\frac{K}{P} + \frac{L(2K-1)(K-1)\bar{\gamma}}{6}\Big) \Big] 
	$$
	$P$ serves as a scaling factor in the denominator, which implies K-AVG scales better than asynchronization.
\end{proof}

The comparison of practical scalability between \AVG and \ASGD implementations is shown in section \ref{experiment:with ASGD}.

\subsection{Optimal $K$ for convergence is not always $1$}
\label{subsection:K}
Unlike convex optimization problems where all learners converge to the same optimum, different learners may converge to different local optimums in nonconvex case. As a consequence, the 
frequency of averaging for nonconvex problems may be different from that of convex cases intuitively. \cite{zhang2016parallel} expressed the same concern, their experimental results showed that periodic averaging 
tends to improve the solution quality. Contrary to popular belief that more frequent averaging i.e. smaller $K$ speeds up convergence, we show that the optimal frequency $K$ for convergence is not always $1$.
We consider the case that the amount of samples processed $N*K$ is constant, which means that the computational time remains as a constant given a fixed number of 
processors. If every other parameter stays the same, larger $K$ means longer delay and fewer updates of global parameter $\widetilde{\bw}_n$. 
The following theorem discusses the impact and optimal choice of $K$ in K-AVG under such an assumption.

\begin{theorem}
	\label{delay}
	Let $S=N*K$ be a constant. Suppose that Algorithm \ref{algorithm1} is run with a fixed stepsize $\gamma_n=\bar{\gamma}$, a fixed batch size $B_{n} = \bar{B}$ satisfying 
	$$
	1 \geq  \frac{L^2\bar{\gamma}^2(K+1)(K-2)}{2} + L\bar{\gamma} K,~~~1-\delta\geq L^2\bar{\gamma}^2
	$$
	with some constant $0<\delta<1$. 
	If 
	\begin{equation}
	\label{condition:optimal_K}
	\frac{(1-\delta)(F(\widetilde{\bw}_1)-F^*)}{S\bar{\gamma}\delta}  > \frac{(3\delta-1)L\bar{\gamma}M}{2\delta P \bar{B}}  + \frac{L^2\bar{\gamma}^2M}{3\bar{B}}.
	\end{equation}
	Then the optimal choice of $K$ is greater than $1$.
\end{theorem}

\begin{proof}
	Under the assumption $S=N*K$, we can rewrite the bound (\ref{fixed2}) as
	\begin{equation}
	\label{bound:K}
	\frac{1}{N}\EE \sum\limits_{j=1}^N \big\| \nabla F(\widetilde{\bw}_j)\big\|_2^2 \leq \Big[\frac{2(F(\widetilde{\bw}_1)-F^*)K}{S(K-1+\delta)\bar{\gamma}} + \frac{LK\bar{\gamma} M}{\bar{B}(K-1+\delta)} \Big(\frac{K}{P} + \frac{L(2K-1)(K-1)\bar{\gamma}}{6}\Big) \Big] .
	\end{equation}
	To move on, we set 
	$$
	B(K) := \Big(\alpha+ \beta K+ \eta(2K-1)(K-1)\Big) \Big(\frac{K}{K-1+\delta}\Big),
	$$
	where 
	$$
	\alpha = \frac{2(F(\widetilde{\bw}_1)-F^*)}{S\bar{\gamma}},~\beta = \frac{L\bar{\gamma}M}{P\bar{B}},~\eta= \frac{L^2\bar{\gamma}^2M}{6\bar{B}}.
	$$
	To minimize the right hand side of (\ref{bound:K}), it is equivalent to solve the following integer program
	$$
	K^* = \argmin\limits_{K \in \NN^*} B(K),
	$$
	which can be very hard. Meanwhile, one should notice that $K^*$ depends on some unknown quantities such as $L$, $M$ and $(F(\widetilde{\bw}_1)-F^*)$. 
	Instead, we investigate the monotonicity of $B(K)$. It is easy to check that $\Big(\alpha+ \beta K+ \eta(2K-1)(K-1)\Big)$ is monotone increasing for all $K\geq 1$, and 
	$\Big(\frac{K}{K-1+\delta}\Big)$ is monotone decreasing for $K \geq 1$. Thus, there exists a unique $K^*$. Then a sufficient condition for $K^*>1$ is that $B(2)<B(1)$, which implies
	$$
	\frac{1-\delta}{2\delta} \alpha > \frac{3\delta-1}{2\delta} \beta +3\eta.
	$$

\end{proof}

The meaning of Theorem \ref{delay} is that it indicates when it comes to nonconvex optimization, more frequent averaging is not necessary.
The sufficient condition (\ref{condition:optimal_K}) implies that larger value of $\big(F(\widetilde{\bw}_1)-F^*\big)$ requires larger $K$ thus longer delay to decrease the bound in (\ref{bound:K}). The intuition is that if the initial weight is too far away from $F^*$, then less frequent synchronizations can lead to faster convergence rate. Less frequent averaging implies higher variance in general. It is quite reasonable to think that if it is still far away from the solution, a stochastic gradient direction with larger variance may be preferred.

As we have already mentioned in the proof, optimal value $K^*$ depends on quantities such as $L$, $M$, and $(F(\widetilde{\bw}_1)-F^*)$ which are unknown to us in practice. 
Therefore, to obtain a concrete $K^*$ in practice is not so realistic.
Note that when $K^*>1$ doesn't necessarily mean that $K^*$ is very close to $1$. In our experiments (see section \ref{section:experiments}), $K^*$ can be as large as $16$ or $32$ in some situation.

\section{Experiments}
\label{section:experiments}

We conduct experiments to validate our analysis on the scalability of
\AVG vs. \ASGD implementations,  the optimal delay in averaging ( optimal value of
$K$),  the convergence comparison with the sequential algorithm,
i.e., \SGD, and the comparison of the learning rates allowed with
\ASGD. 

In our application gradient
descent is implemented with Torch, and the communication is
implemented using CUDA-aware openMPI 2.0 through the mpiT
library.  All implementations use the cuDNN library for forward propagation and backward propagation. Our experiments are done on a cluster of 32 Minsky
nodes interconnected with Infiniband. Each node is an IBM S822LC
system containing 2 Power8 CPUs with 10
cores each, and 4 NVIDIA Tesla P100 GPUs.

We experiment with the \cifar \cite{krizhevsky2009learning} data set using
the \vgg and \nin models. \cifar 
contains $50,000$ training images and $10,000$ test images, each
associated with $1$ out of $10$ possible labels. 

\subsection{Comparison with \ASGD}
\label{experiment:with ASGD}
Theorem~\ref{scalability} shows that
the convergence bound of \AVG is not affected much by the scaling of
$P$ but rather by $K$. On the contrary, the convergence bound of \ASGD
increases linearly with $P$. Thus we expect poorer convergence behavior
of \ASGD implementations at large $P$ in comparison with \AVG. 

Figures~\ref{fig:p-accuracy-vgg} and ~\ref{fig:p-accuracy-nin}  compare the performance of \AVG with two \ASGD
implementations, \downpour and \eamsgd, for two neural networks, \vgg \cite{simonyan2014very}
and \nin \cite{lin2013network}, respectively.  We use $P=8, 16, 32, 64$ and
$128$ learners, and show the test accuracies.  All implementations use
the same initial learning rate ($\gamma_0 = 1$)  and learning rate
adaptation schedule (reduce $\gamma$ by half after 50 epochs). The batch size is fixed at $\bar{B}=16$. We run
for 600 epochs with $K=16$ for \AVG.  

In both figures, \AVG always achieves better
test accuracy than \downpour and \eamsgd. The test accuracies for
\downpour and \eamsgd decreases as $P$ increases (the effect is more
pronounced for \vgg).  When $P$ reaches 128, the
accuracies of \downpour and \eamsgd both degrade to around 10\%, i.e,
random guess. The \ASGD implementations do not converge with $\gamma_0
= 1$ at $P=128$. 

\begin{figure}[!htb]
	\begin{minipage}{0.5\textwidth}
		\includegraphics[width=0.95\linewidth]{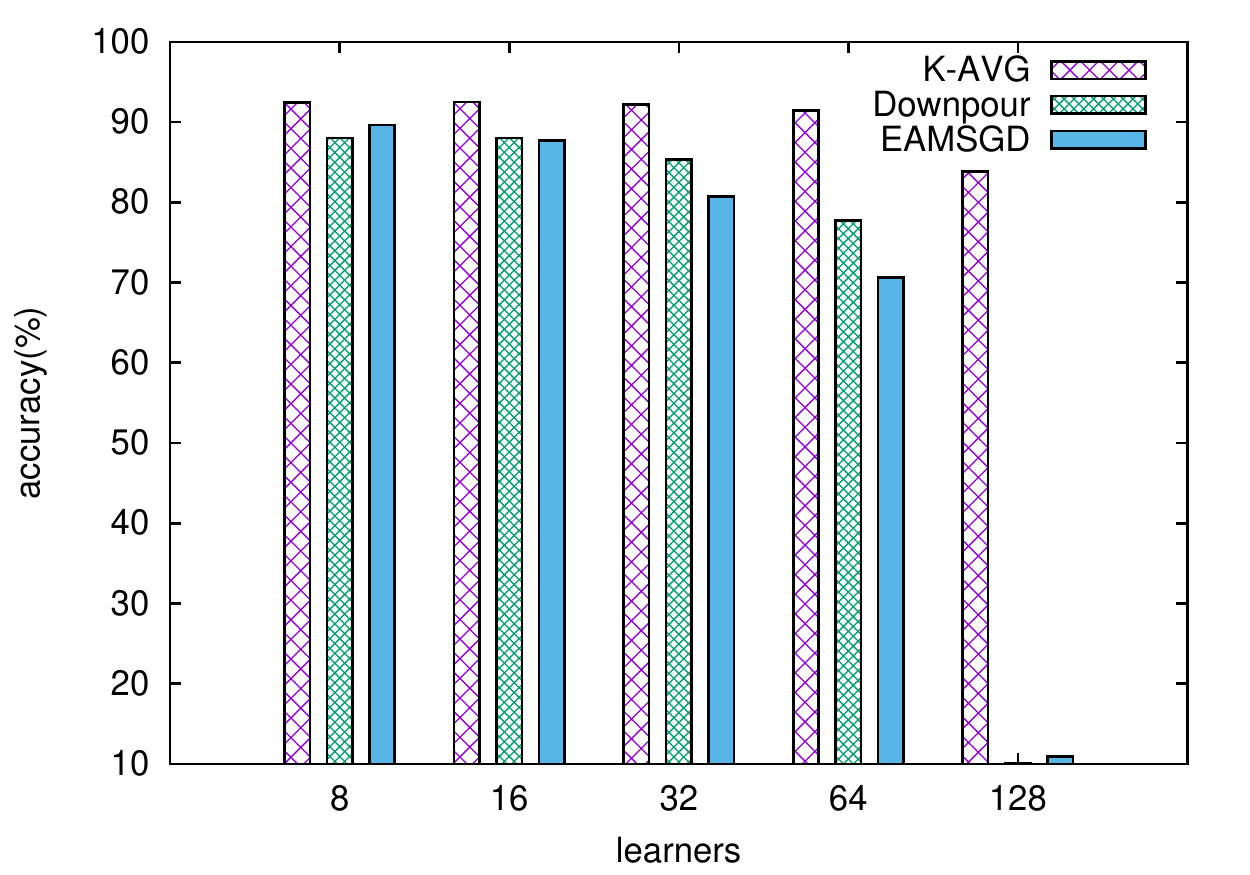}
		\caption{Scaling with vgg}
		\label{fig:p-accuracy-vgg}
	\end{minipage}%
	\begin{minipage}{0.5\textwidth}
		\includegraphics[width=0.95\linewidth]{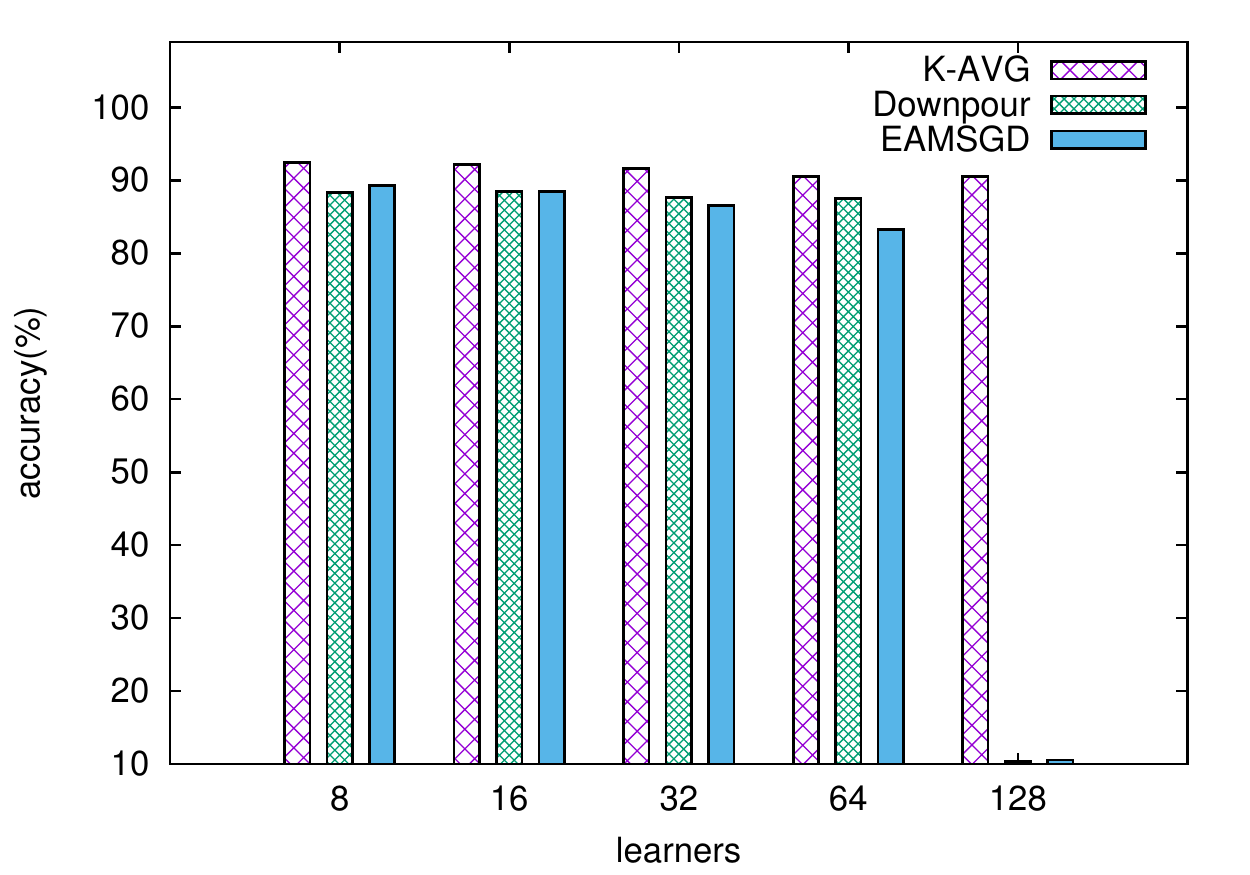}
		\caption{Scaling with nin}
		\label{fig:p-accuracy-nin}
	\end{minipage}
\end{figure}

When we set the initial learning rate $\gamma_0 = 0.1$,  \downpour
achieves around 80\% and 87\% test accuracies for \vgg and \nin,
respectively; \eamsgd still does not converge.

\AVG achieves better test accuracy than \ASGD implementations,
and in our experiment it is also faster.  We measure wall-clock times for
all implementations after 600 epochs.  Naturally for \AVG, $K$ impacts
the ratio of communication vs. computation.  We still use $K=16$. 

Figures \ref{fig:p-speedup-vgg} and \ref{fig:p-speedup-nin} show the speedups of \AVG over \downpour and
\eamsgd, with \vgg and \nin, respectively.  In
Figure \ref{fig:p-speedup-vgg},  when $P=8$, the \ASGD implementations are slightly faster than
\AVG. As $P$ increases, the speedup increases.  When $P=128$, the
speedups are around 2.5 and 2.6 over \downpour and \eamsgd,
respectively. Similar behavior is observed in
Figure \ref{fig:p-speedup-nin} for \nin. 

\begin{figure}
	\begin{minipage}{0.5\textwidth}
		\includegraphics[width=0.95\textwidth]{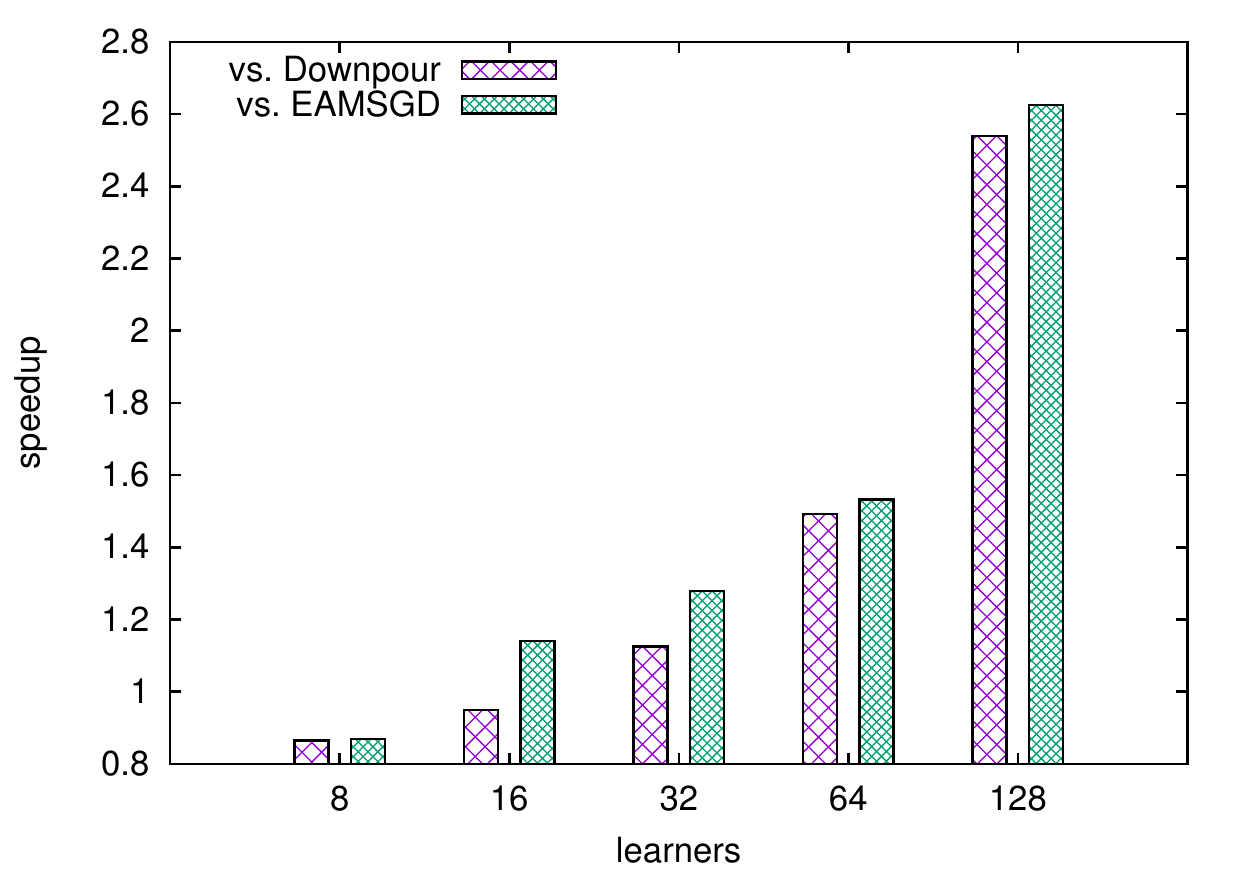}
		\caption{Speedup of \AVG over \ASGD implementations with vgg}
		\label{fig:p-speedup-vgg}
	\end{minipage}
	\begin{minipage}{0.5\textwidth}
		\includegraphics[width=0.95\textwidth]{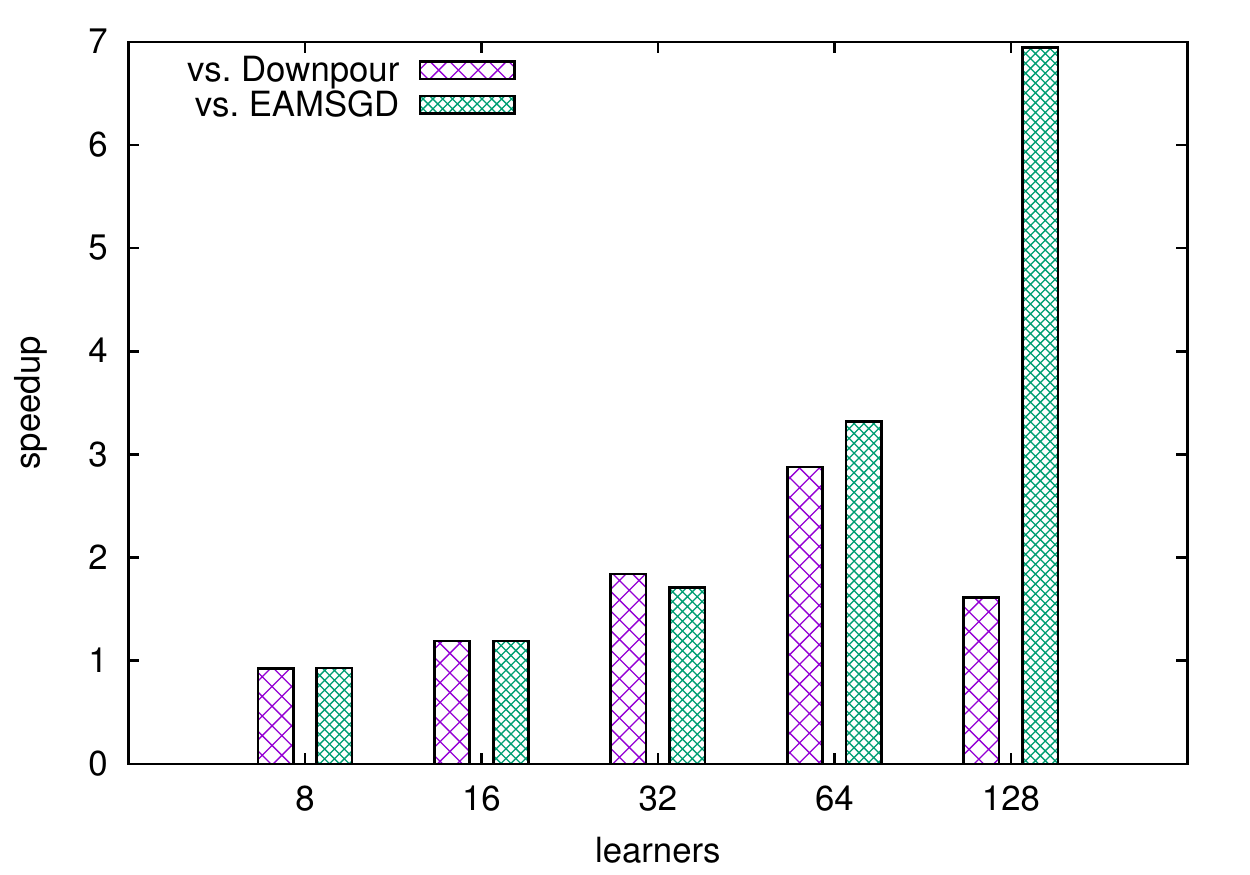}
		\caption{Speedup of \AVG over \ASGD implementations with nin}
		\label{fig:p-speedup-nin}
	\end{minipage}
\end{figure}

\subsection{The optimal delay in averaging for \AVG}
\label{experiment:K}

For \AVG, $K$ regulates its behavior.  From the execution time
perspective, larger $K$ results in fewer communications to process a
given number of data samples.  From the convergence perspective,
smaller $K$ reduces the variances among learners.  In the extreme
case where $K=1$, \AVG is equivalent to synchronous parallelization of
\SGD. People tend to think that smaller $K$ results in faster
convergence in terms of number of data samples processed.  As
discussed in Section \ref{subsection:K}, there are scenarios where
$K_{opt}$ is not 1.

We evaluate the convergence behavior of \AVG with different $K$ values
for \vgg and \nin. We experiment with $K=1, 2, 4, 16, 32$, and $64$.
Figures \ref{fig:K-vgg} and \ref{fig:K-nin} show the test accuracies achieved after 600 epochs
for $P=8, 16, 32, 64$, and $128$, for \vgg and \nin, respectively.
Again we use the initial $\gamma_0$=1, and after every 50 epochs, $\gamma$
is reduced by half. The batch size is fixed as $\bar{B}=32$.

In Figure \ref{fig:K-vgg}, strikingly, none of the experiments show the
optimal value of $K$ for \AVG is 1. $K_{opt}$ ranges from 32 (when $P=8$) farthest away
from 1, to 2 (when $P=64$),  closest to 1.  In this set of experiments, as $P$ increases, $K_{opt}$
tends to decrease.  Also with smaller $P$, \AVG is more forgiving in
terms of the choices of $K$. For example, when $P=8$, test accuracies
for different $K$ are similar. With larger $P$, however, choosing
a $K$ that is too large has severe punishing consequences.  For
example, when $P=128$, $K_{opt}=4$, and the test accuracy degrades rapidly with the increase of
$K$ beyond 4.  

In Figure \ref{fig:K-nin}, almost all experiments show $K_{opt}$=1.  The
exception is with $p=8$, and the accuracy is slightly higher (by
0.27\%) at $K=8$ than $K=1$.  Again we see for small $p$ the choices
of $K$ is not critical, while for large $p$ the degradation in
accuracy is rapid with the increase of $K$ beyond $K_{opt}$. 

\begin{figure}
	\begin{minipage}{0.5\textwidth}
		\includegraphics[width=0.95\textwidth]{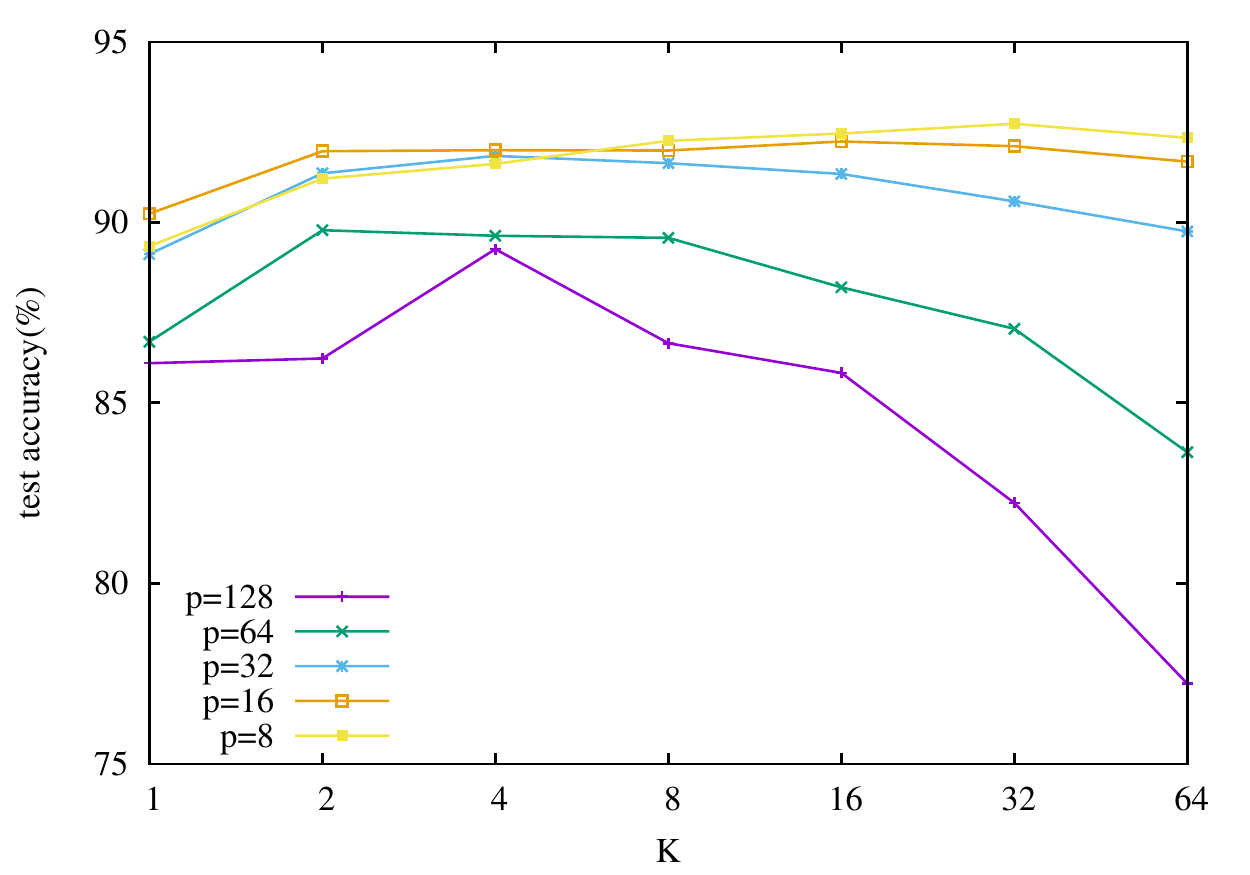}
		\caption{Test accuracy with different $K$}
		\label{fig:K-vgg}
	\end{minipage}
	\begin{minipage}{0.5\textwidth}
		\includegraphics[width=0.95\textwidth]{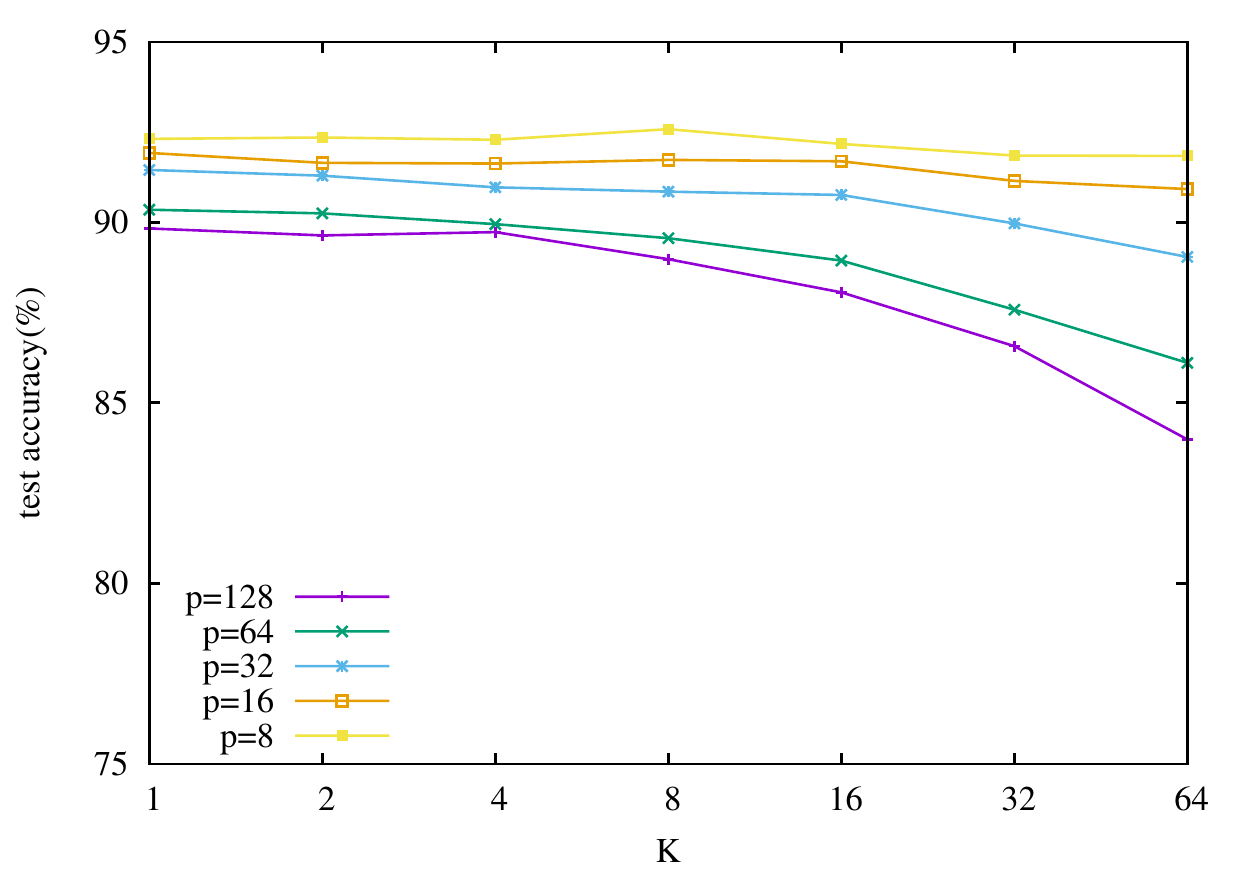}
		\caption{Test accuracy with different $K$}
		\label{fig:K-nin}
	\end{minipage}
\end{figure}

Since in our experiments the same hyper-parameters are used for \vgg
and \nin,  it is likely that the Lipschitz constant $L$ largely
determines the differences in $K_{opt}$ between the two cases. 

\subsection{Convergence comparison with \SGD}
We compare the performance of \AVG against the sequential
implementation, that is, \SGD.  We evaluate the test accuracies achieved
and the wall clock time used. 

Figure \ref{fig:avg-seq-gap} shows the
accuracy gap between \AVG and SGD. With $8$ and $16$
learners, \AVG is slightly worse than sequential \SGD for \vgg but better than \nin.  \AVG and \SGD
achieve comparable accuracies with
32 learners.  As the number of learners reaches 64
and 128, significant accuracy degradation, up to 8.8\% is observed
for \vgg. Interestingly, the accuracy degradation for
\nin is still within 1.3\% with 128 learners. 

\begin{figure}
	\begin{minipage}{0.5\textwidth}
		\includegraphics[width=0.95\textwidth]{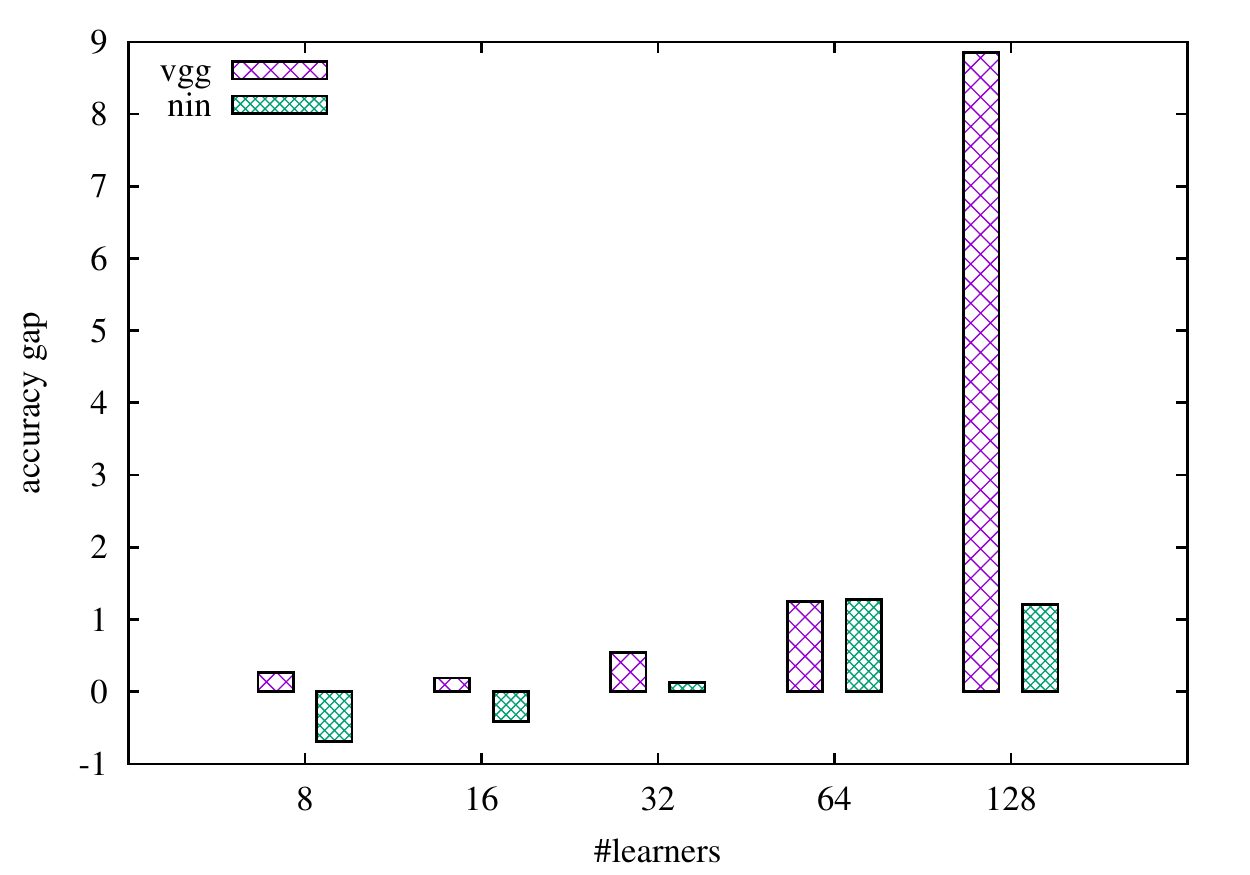}
		\caption{Accuracy gap}
		\label{fig:avg-seq-gap}
	\end{minipage}
	\begin{minipage}{0.5\textwidth}
		\includegraphics[width=0.95\textwidth]{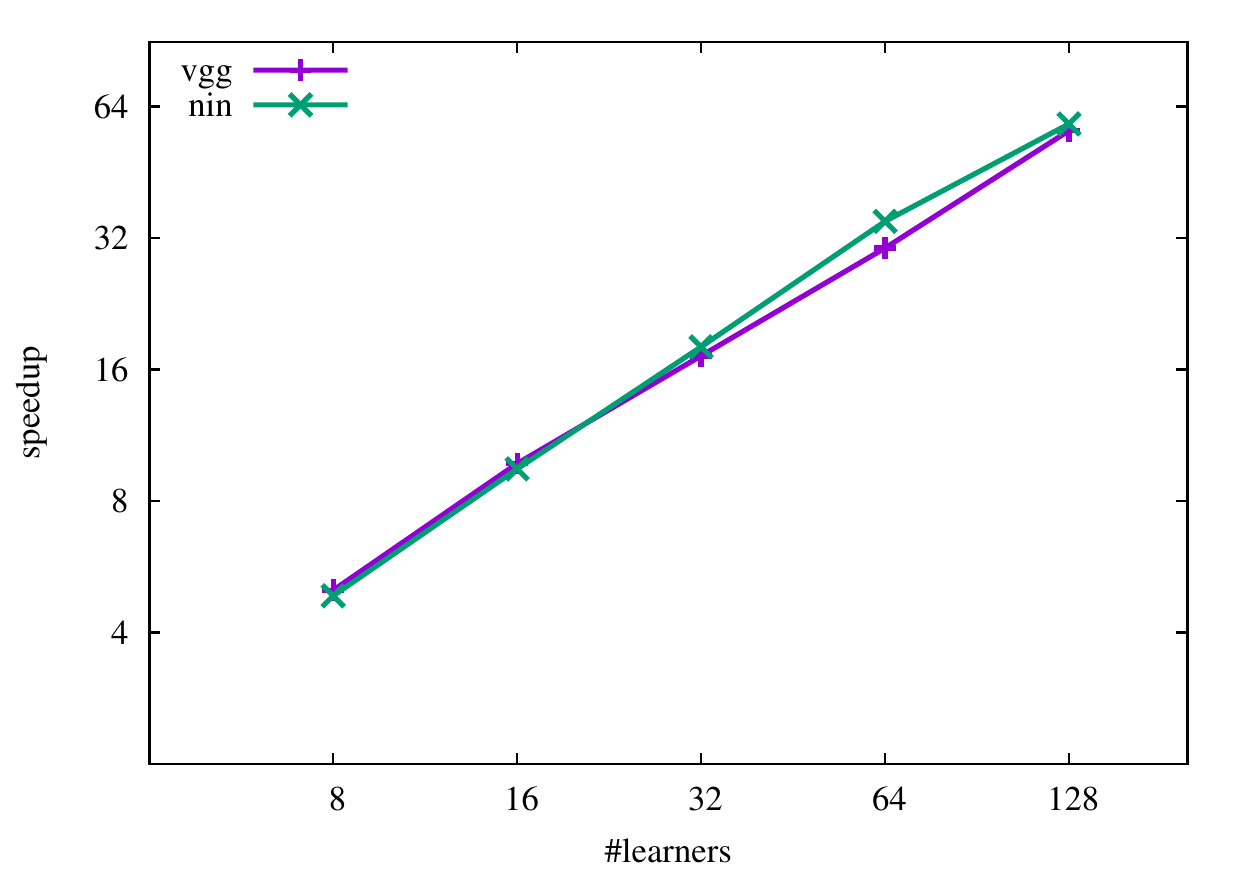}
		\caption{Epoch time speedup}
		\label{fig:epoch-speedup}
	\end{minipage}
\end{figure}

Figure \ref{fig:epoch-speedup} shows the epoch time speedup of \AVG against SGD
with 8, 16, 32, 64, and 128 learners.  With 8 learners, the speedups
for \vgg and \nin are 4.9, and 4.8, respectively. With
128 learners, the speedups for \vgg and \nin are 56.3 and
58.3, respectively. Since twice as
many epochs are run for \AVG in comparison to \SGD, the wallclock
time speedups are half of these numbers. This means linear speedup is achieved.

\section{Conclusion}
In this paper, we adopt and analyze \AVG for solving large scale machine learning problems with nonconvex objectives.
We establish the convergence results of \AVG under both fixed and diminishing stepsize, and show that with a properly chosen sequence of stepsizes,
\AVG achieves similar convergence rate consistent with its sequential counterpart. We show that 
\AVG scales better than \ASGD with properly chosen $K$ when $P$ is large. \AVG allows larger stepsizes that still guarantees convergence while \ASGD may fail to converge. Contrary to popular belief, we show that the length of delay to average learning parameters among parallel learners is 
not necessarily to be $1$. Although, a proper choice of $K_{opt}$ remains unknown, we analytically
explain the dependence of $K_{opt}$ on other parameters which we hope can serve as a users' guide in practical implementations. 

\section{Proofs}
\subsection{Proof of Theorem \ref{theorem:fixed stepsize}}
\label{proof:theorem fixed}
\begin{proof}
	We denote $\widetilde{\bw}_{\alpha}$ as the $\alpha$-th global update in \AVG, and denote $\bw^j_{\alpha+t}$ as $t$-th local update on processor $j$.
	The $(\alpha+1)$-th global average can be written as
	$$
	\widetilde{\bw}_{\alpha+1} = \frac{1}{P}\sum\limits_{j=1}^P  \bw_{\alpha+K}^{j} = \frac{1}{P}\sum\limits_{j=1}^P\Big[\bw_{\alpha}^j - \sum\limits_{t=0}^{K-1}\frac{\gamma_t}{B}\sum\limits_{s=1}^B\nabla F(\bw_{\alpha+t}^j;\xi_{t,s}^j) \Big],
	$$
	according to our algorithm, the random variables $\xi_{t,s}^j$ are i.i.d. for all $t=0,...,K-1$, $s=1,...,B$, and $j=1,...,P$. 
	\begin{align}
	F(\widetilde{\bw}_{\alpha+1}) -F(\widetilde{\bw}_{\alpha}) &\leq \Big\langle\nabla F(\widetilde{\bw}_{\alpha}), \widetilde{\bw}_{\alpha+1}-\widetilde{\bw}_{\alpha} \Big\rangle+\frac{L}{2}\big\|\widetilde{\bw}_{\alpha+1} - \widetilde{\bw}_{\alpha}\big\|_2^2 \\
	& \leq -\Big\langle \nabla F(\widetilde{\bw}_{\alpha}), \frac{1}{P}\sum\limits_{j=1}^P\sum\limits_{t=0}^{K-1}\frac{\gamma_t}{B}\sum\limits_{s=1}^B \nabla F(\bw_{\alpha+t}^j;\xi_{t,s}^j) \Big\rangle\\
	& +\frac{L}{2} \Big\|\frac{1}{P}\sum\limits_{j=1}^P\sum\limits_{t=0}^{K-1}\frac{\gamma_t}{B}\sum\limits_{s=1}^B\nabla F(\bw_{\alpha+t}^j;\xi_{t,s}^j) \Big\|_2^2.
	\end{align}
	Under the assumption that a constant stepsize is implemented within each inner parallel step, i.e. $\gamma_t = \gamma$, for $t=0,...,K-1$, one can immediately simplify the above inequality as 
	\begin{align}
	F(\widetilde{\bw}_{\alpha+1}) -F(\widetilde{\bw}_{\alpha}) &\leq - \frac{\gamma}{PB}\sum\limits_{j=1}^P\sum\limits_{t=0}^{K-1}\sum\limits_{s=1}^B\Big\langle \nabla F(\widetilde{\bw}_{\alpha}),\nabla F(\bw_{\alpha+t}^j;\xi_{t,s}^j) \Big\rangle	\label{rhs1}\\
	&+ \frac{L\gamma^2}{2P^2B^2} \Big\| \sum\limits_{j=1}^P\sum\limits_{t=0}^{K-1}\sum\limits_{s=1}^B \nabla F(\bw_{\alpha+t}^j;\xi_{t,s}^j) \Big\|_2^2 .\label{rhs2}
	\end{align}
	The goal here is to investigate the expectation of $ F(\widetilde{\bw}_{\alpha+1}) -F(\widetilde{\bw}_{\alpha}) $ over all random variables $\xi_{t,s}^j$.
	Under the unbiased estimation Assumption \ref{unbias}, by taking the overall expectation we can immediately get 
	$$
	\EE \frac{1}{B} \sum\limits_{s=1}^B\nabla F(\bw_{\alpha+t}^j;\xi_{t,s}^j ) =	\EE \Big[ \frac{1}{B} \sum\limits_{s=1}^B \EE_{\xi^j_{t,s}}\nabla F(\bw_{\alpha+t}^j;\xi_{t,s}^j | \bw_{\alpha+t}^j)\Big]
	= \EE\nabla F(\bw_{\alpha+t}^j).
	$$
	for fixed $j$ and $t$.
	Here we abuse the expectation notation $\EE$ a little bit. Throughout this proof, $\EE$ always means taking the overall expectation. We will frequently use the above iterative conditional expectation trick in the following analysis.
	As a result, we can drop the summation over $s$ due to an averaging factor $B$ in the dominator of (\ref{rhs1}). 
	Next, we show how to get rid of the summation over $j$. Recall that $\bw^j_{\alpha+1} = \widetilde{\bw}_{\alpha} - \frac{\gamma}{B} \sum_{s=1}^B\nabla F(\widetilde{\bw}_{\alpha};\xi^j_{0,s})$. Obviously, $\bw_{\alpha+1}^j$, $j=1,...,P$ are i.i.d. for fixed $\widetilde{\bw}_{\alpha}$ because $\xi_{0,s}^j$, $j=1,...,P$, $s=1,...,B$ are i.i.d. Similarly,  $\bw^j_{\alpha+2} = \bw^j_{\alpha+1} - \frac{\gamma}{B} \sum_{s=1}^B\nabla F(\bw^j_{\alpha+1};\xi^j_{1,s})$, $j=1,...,P$ are i.i.d. due to the fact that 
	$\bw_{\alpha+1}^j$'s  are i.i.d., $\xi_{1,s}^j$'s are i.i.d., and $\bw_{\alpha+1}^j$'s are independent from $\xi_{1,s}^j$'s. By induction, one can easily show that for each fixed $t$, $\bw_{\alpha+t}^j$, $j=1,...,P$ are i.i.d. Thus for each fixed $t$
	$$
	\frac{1}{P} \sum\limits_{j=1}^P \EE \nabla F(\bw_{\alpha+t}^j) = \EE \nabla F(\bw_{\alpha+t}^j).
	$$
	We can therefore get rid of the summation over $j$ as well in (\ref{rhs1}). By taking the overall expectation on both sides of (\ref{rhs1}) and (\ref{rhs2}), we have
	\begin{align}
	\EE F(\widetilde{\bw}_{\alpha+1}) -F(\widetilde{\bw}_{\alpha}) &\leq - \gamma\sum\limits_{t=0}^{K-1}\EE\Big\langle \nabla F(\widetilde{\bw}_{\alpha}),\nabla F(\bw_{\alpha+t}^j) \Big\rangle	\label{newrhs1}\\
	&+ \frac{L\gamma^2}{2P^2B^2} \EE \Big\| \sum\limits_{j=1}^P\sum\limits_{t=0}^{K-1}\sum\limits_{s=1}^B \nabla F(\bw_{\alpha+t}^j;\xi_{t,s}^j) \Big\|_2^2 . \label{newrhs2}
	\end{align}
	We are going to bound (\ref{newrhs1}) and (\ref{newrhs2}) respectively.
	Obviously (we treat $\widetilde{\bw}_{\alpha}$ as a constant vector at this moment), 
	\begin{equation}
	\label{proof:crossterm}
	\begin{aligned}
	- \gamma\sum\limits_{t=0}^{K-1}\EE\Big\langle \nabla F(\widetilde{\bw}_{\alpha}),\nabla F(\bw_{\alpha+t}^j) \Big\rangle& =  -\frac{\gamma}{2} \sum\limits_{t=0}^{K-1}\Big( \big\| \nabla F(\widetilde{\bw}_{\alpha})\big\|_2^2
	+\EE \big\| \nabla F(\bw_{\alpha+t}^j)\big\|_2^2 \Big)\\
	& + \frac{\gamma}{2} \sum\limits_{t=0}^{K-1} \EE \big\| \nabla F(\bw_{\alpha+t}^j) - \nabla F(\widetilde{\bw}_{\alpha})\big\|_2^2 \\
	& \leq -\frac{(K+1)\gamma}{2} \big\| \nabla F(\widetilde{\bw}_{\alpha})\big\|_2^2 -\frac{\gamma}{2} \sum\limits_{t=1}^{K-1}\EE \big\| \nabla F(\bw_{\alpha+t}^j)\big\|_2^2 \\
	& + \frac{L^2\gamma}{2} \sum\limits_{t=1}^{K-1}  \EE \big\| \bw_{\alpha+t}^j -\widetilde{\bw}_{\alpha}\big\|_2^2 ,
	\end{aligned}
	\end{equation}
	where we used the fact that $\widetilde{\bw}_{\alpha}^j=\widetilde{\bw}_{\alpha}$, for $j=1,...,P$ for the last term and the assumption that gradient $\nabla F$ is Lipschitz continuous.
	Note that
	\begin{align*}
	\EE \big\| \bw_{\alpha+t}^j -\widetilde{\bw}_{\alpha}\big\|_2^2 &= \frac{\gamma^2}{B^2} \EE \big\| \sum\limits_{i=0}^{t-1}\sum\limits_{s=1}^B   \nabla F(\bw_{\alpha+i}^j ; \xi_{i,s}^j)  \big\|_2^2
	\leq \frac{t\gamma^2}{B^2}\EE  \sum\limits_{i=0}^{t-1}\big\|\sum\limits_{s=1}^B  \nabla F(\bw_{\alpha+i}^j ; \xi_{i,s}^j)  \big\|_2^2\\
	& = \frac{t\gamma^2}{B^2}\EE  \sum\limits_{i=0}^{t-1}\big\|\sum\limits_{s=1}^B  \big( \nabla F(\bw_{\alpha+i}^j ; \xi_{i,s}^j) - \nabla F(\bw_{\alpha+i}^j )+ \nabla F(\bw_{\alpha+i}^j )  \big)\big\|_2^2 \\
	& = \frac{t\gamma^2}{B^2}\EE  \sum\limits_{i=0}^{t-1}\big\|\sum\limits_{s=1}^B  \big( \nabla F(\bw_{\alpha+i}^j ; \xi_{i,s}^j) - \nabla F(\bw_{\alpha+i}^j ) \big)\big\|_2^2 
	+ t\gamma^2 \EE  \sum\limits_{i=0}^{t-1} \big\| \nabla F(\bw_{\alpha+i}^j ) \big)\big\|_2^2  \\
	& + \frac{t\gamma^2}{B^2} 2\EE  \sum\limits_{i=0}^{t-1} \EE_{\xi_{i,*}}^j\Big\langle \sum\limits_{s=1}^B  \big( \nabla F(\bw_{\alpha+i}^j ; \xi_{i,s}^j) - \nabla F(\bw_{\alpha+i}^j ) \big), B \nabla F(\bw_{\alpha+i}^j ) \Big\rangle \\
	& = \frac{t\gamma^2}{B^2} \EE  \sum\limits_{i=0}^{t-1}\sum\limits_{s=1}^B  \big\| \nabla F(\bw_{\alpha+i}^j ; \xi_{i,s}^j) - \nabla F(\bw_{\alpha+i}^j ) \big\|_2^2 
	+ t\gamma^2 \EE  \sum\limits_{i=0}^{t-1} \big\| \nabla F(\bw_{\alpha+i}^j ) \big\|_2^2  \\
	& \leq \frac{t^2\gamma^2 M}{B} +  t\gamma^2 \EE  \sum\limits_{i=0}^{t-1} \big\| \nabla F(\bw_{\alpha+i}^j ) \big\|_2^2 
	\label{expectation:w}
	\end{align*}
	where the first inequality is due to Cauchy-Schwartz inequality, the last equity is due to Assumption \ref{unbias} and the last inequality is due to Assumption \ref{variance}.
	We plug the above inequality back into (\ref{proof:crossterm})
	and get
	\begin{equation}
	\label{boundrhs1}
	\begin{aligned}
	- \gamma\sum\limits_{t=0}^{K-1}\EE\Big\langle \nabla F(\widetilde{\bw}_{\alpha}),\nabla F(\bw_{\alpha+t}^j) \Big\rangle &\leq -\frac{(K+1)\gamma}{2}\Big(1-\frac{L^2\gamma^2K(K-1)}{2(K+1)} \Big) \big\| \nabla F(\widetilde{\bw}_{\alpha})\big\|_2^2 \\
	& -\frac{\gamma}{2}\Big( 1- \frac{L^2\gamma^2(K+1)(K-2)}{2}\Big) \sum\limits_{t=1}^{K-1}\EE \big\| \nabla F(\bw_{\alpha+t}^j)\big\|_2^2   \\
	&+ \frac{L^2\gamma^3M(2K-1)K(K-1)}{12B} \\
	\end{aligned}
	\end{equation}
	
	On the other hand, to bound (\ref{newrhs2}), we can apply the similar analysis,
	\begin{equation}
	\begin{aligned}
	&\frac{L\gamma^2}{2P^2B^2} \EE \Big\| \sum\limits_{j=1}^P\sum\limits_{t=0}^{K-1}\sum\limits_{s=1}^B \nabla F(\bw_{\alpha+t}^j;\xi_{t,s}^j) \Big\|_2^2 \\
	& \leq \frac{LK\gamma^2}{2P^2B^2} \EE \sum\limits_{t=0}^{K-1}\Big\| \sum\limits_{j=1}^P\sum\limits_{s=1}^B \nabla F(\bw_{\alpha+t}^j;\xi_{t,s}^j) \Big\|_2^2 \\	
	& \leq \frac{LK\gamma^2}{2P^2B^2} \EE \sum\limits_{t=0}^{K-1}\Big\| \sum\limits_{j=1}^P\sum\limits_{s=1}^B \Big(\nabla F(\bw_{\alpha+t}^j;\xi_{t,s}^j)  
	-\nabla F(\bw_{\alpha+t}^j)+ \nabla F(\bw_{\alpha+t}^j)\Big)\Big\|_2^2 \\	
	& \leq \frac{LK^2\gamma^2M}{2PB} + \frac{LK\gamma^2}{2}\sum\limits_{t=0}^{K-1}  \EE\Big\|\nabla F(\bw_{\alpha+t}^j)\Big\|_2^2\\
	& \leq \frac{LK^2\gamma^2M}{2PB} +  \frac{LK\gamma^2}{2}\sum\limits_{t=0}^{K-1}\EE\Big\|\nabla F(\bw_{\alpha+t}^j)\Big\|_2^2 \label{boundrhs2}. 
	\end{aligned}
	\end{equation}
	Combine the results in (\ref{boundrhs1}) and (\ref{boundrhs2}), we have 
	\begin{align*}
	\EE F(\widetilde{\bw}_{\alpha+1}) -F(\widetilde{\bw}_{\alpha}) & \leq -\frac{(K+1)\gamma}{2}\Big(1-\frac{L^2\gamma^2K(K-1)}{2(K+1)} -\frac{L\gamma K}{(K+1)}  \Big) \big\| \nabla F(\widetilde{\bw}_{\alpha})\big\|_2^2 \\
	&  -\frac{\gamma}{2}\Big( 1- \frac{L^2\gamma^2(K+1)(K-2)}{2} -L\gamma K\Big)\sum\limits_{t=1}^{K-1}\EE\Big\|\nabla F(\bw_{\alpha+t}^j)\Big\|_2^2 \\
	& + \frac{L^2\gamma^3M(2K-1)K(K-1)}{12B} + \frac{LK^2\gamma^2M}{2PB}
	\end{align*}
	Under the condition $1 \geq  \frac{L^2\gamma^2(K+1)(K-2)}{2} + L\gamma K$, the second term on the right hand side can be discarded. This condition also implies that
	$$
	\frac{(K+1)\gamma}{2}\Big(1-\frac{L^2\gamma^2K(K-1)}{2(K+1)} -\frac{L\gamma K}{(K+1)}  \Big) \geq \frac{\gamma(K- L^2\gamma^2)}{2} .
	$$
	Then, together with the condition $1-\delta \geq L^2 \gamma^2$ for some $0<\delta<1$, we have
	\begin{align*}
	\EE F(\widetilde{\bw}_{\alpha+1}) - F(\widetilde{\bw}_{\alpha})  \leq -\frac{(K-1+\delta)\gamma}{2} \big\| \nabla F(\widetilde{\bw}_{\alpha})\big\|_2^2  + \frac{L\gamma^2KM}{2B} \Big(\frac{K}{P} + \frac{L(2K-1)(K-1)\gamma}{6}\Big) .
	\end{align*}
	If we allow both batch size and step size to change after each averaging step, by taking the summation we have
	\begin{align*}
	\EE F(\widetilde{\bw}_{N}) - F(\widetilde{\bw}_{1})  \leq \sum\limits_{j=1}^N  -\frac{(K-1+\delta)\gamma_j}{2} \big\| \nabla F(\widetilde{\bw}_{\alpha})\big\|_2^2  + \frac{L\gamma_j^2KM}{2B_j} \Big(\frac{K}{P} + \frac{L(2K-1)(K-1)\gamma_j}{6}\Big) .
	\end{align*}
	Under Assumption \ref{lowerbound}, we have
	\begin{equation}
	F^* - F(\widetilde{\bw}_1) \leq F(\widetilde{\bw}_{n}) -F(\widetilde{\bw}_1),
	\end{equation}
	Combining both, we can immediately get the following bound
	\begin{equation}
	\EE  \sum\limits_{j=1}^N \gamma_j \big\| \nabla F(\widetilde{\bw}_j)\big\|_2^2 \leq \frac{ 2(F(\widetilde{\bw}_1)-F^*)}{K-1+\delta} + \sum\limits_{j=1}^N\frac{L\gamma_j^2KM}{(K-1+\delta)B_j} 
	\Big(\frac{K}{P} + \frac{L(2K-1)(K-1)\gamma_j}{6}\Big) . \label{general_bound}
	\end{equation}
	If we employ a constant step size and batch size,  we get the bound on the expected average squared gradient norms of $F$ as following:
	$$
	\frac{1}{N}\EE \sum\limits_{j=1}^N \big\| \nabla F(\widetilde{\bw}_j)\big\|_2^2 \leq \Big[\frac{2(F(\widetilde{\bw}_1)-F^*)}{N(K-1+\delta)\bar{\gamma}} + \frac{LK\bar{\gamma} M}{\bar{B}(K-1+\delta)} \Big(\frac{K}{P} + \frac{L(2K-1)(K-1)\bar{\gamma}}{6}\Big) \Big]  ,
	$$
	which completes the proof.
\end{proof}

\subsection{Proof of Theorem \ref{theorem:dynamicsize}}
\label{proof:theorem dynamic}
\begin{proof}
	If we use a diminishing step size $\gamma_j$ and growing batch size $B_j$, 
	Thus, from (\ref{general_bound}), we divide both sides with $\sum_{j=1}^N \gamma_j$, we have
	\begin{align*}
	&\EE  \sum\limits_{j=1}^N \frac{\gamma_j}{\sum_{j=1}^N \gamma_j}\big\| \nabla F(\widetilde{\bw}_j)\big\|_2^2 \\
	& \leq \frac{ 2(F(\widetilde{\bw}_1)-F^*)}{(K-1+\delta)\sum_{j=1}^N \gamma_j} + \sum\limits_{j=1}^N\frac{LK\gamma_j^2M}{B_j\sum_{j=1}^N \gamma_j(K-1+\delta)} \Big(\frac{K}{P} + \frac{L(2K-1)(K-1)\gamma_j}{6}\Big).
	\end{align*}
	If the following restriction of step size is satisfied,
	$$
	\lim\limits_{N\rightarrow \infty} \sum\limits_{j=1}^{N} \gamma_j = \infty,~~ \lim\limits_{N\rightarrow \infty} \sum\limits_{j=1}^N\frac{K\gamma_j^2}{PB_j} < \infty,~~\lim\limits_{N\rightarrow \infty} \sum\limits_{j=1}^N \gamma_j^3 \leq \infty,
	$$
	it immediately implies the convergence of $ \EE  \sum_{j=1}^N \frac{\gamma_j}{\sum_{j=1}^N \gamma_j}\big\| \nabla F(\widetilde{\bw}_j)\big\|_2^2$ when $N\rightarrow \infty$.
\end{proof}

\subsection{Proof of Corollary \ref{corollary}}
\label{proof:corollary}
\begin{proof}
	At first, we assume that $K/P > L(2K-1)(K-1)\bar{\gamma}/6$. Then we can rewrite the bound (\ref{fixed2}) as 
	$$
	\frac{1}{N}\EE \sum\limits_{j=1}^N \big\| \nabla F(\widetilde{\bw}_j)\big\|_2^2 \leq \Big[\frac{ 2(F(\widetilde{\bw}_1)-F^*)}{N(K-1+\delta)\bar{\gamma}} + \frac{2LK^2\bar{\gamma} M}{P\bar{B}(K-1+\delta)}   \Big] 
	$$	
	Set 
	$$
	f(\bar{\gamma}) = \frac{ 2(F(\widetilde{\bw}_1)-F^*)}{N(K-1+\delta)\bar{\gamma}} + \frac{2LK^2\bar{\gamma} M}{P\bar{B}(K-1+\delta)} .
	$$
	By taking $f'=0$, one immediately get 
	$$
	\bar{\gamma}_1 = \sqrt{\frac{(F(\widetilde{\bw}_1)-F^*)\bar{B}P}{K^2LM}}*\frac{1}{\sqrt{N}}.
	$$
	and 
	$$
	f(\bar{\gamma}_1) = \sqrt{\frac{(F(\widetilde{\bw}_1)-F^*)LM}{\bar{B}P}}*\frac{4K}{(K-1+\delta)\sqrt{N}}.
	$$
	By plugging in the value of $\bar{\gamma}_1$, $K/P > L(2K-1)(K-1)\bar{\gamma}/6$ implies that 
	$$
	N> \frac{(F(\widetilde{\bw}_1)-F^*)LBP^3(2K-1)(K-1)}{6K^4M}.
	$$
	At last, we need to check condition (\ref{cond:stepsize}) is hold. It is sufficient to set  
	$$
	N > \frac{L(F(\widetilde{\bw}_1)-F^*)\bar{B}P}{M}.
	$$

	This completes the proof of Corollary \ref{corollary}.
\end{proof}

\bibliographystyle{plainnat}
\bibliography{refer}

\begin{thebibliography}{25}
\providecommand{\natexlab}[1]{#1}
\providecommand{\url}[1]{\texttt{#1}}
\expandafter\ifx\csname urlstyle\endcsname\relax
  \providecommand{\doi}[1]{doi: #1}\else
  \providecommand{\doi}{doi: \begingroup \urlstyle{rm}\Url}\fi

\bibitem[Bottou(1998)]{bottou1998online}
L{\'e}on Bottou.
\newblock Online learning and stochastic approximations.
\newblock \emph{On-line learning in neural networks}, 17\penalty0 (9):\penalty0
  142, 1998.

\bibitem[Bottou et~al.(2016)Bottou, Curtis, and
  Nocedal]{bottou2016optimization}
L{\'e}on Bottou, Frank~E Curtis, and Jorge Nocedal.
\newblock Optimization methods for large-scale machine learning.
\newblock \emph{arXiv preprint arXiv:1606.04838}, 2016.

\bibitem[Chen et~al.(2016)Chen, Pan, Monga, Bengio, and
  Jozefowicz]{chen2016revisiting}
Jianmin Chen, Xinghao Pan, Rajat Monga, Samy Bengio, and Rafal Jozefowicz.
\newblock Revisiting distributed synchronous sgd.
\newblock \emph{arXiv preprint arXiv:1604.00981}, 2016.

\bibitem[Chung(1954)]{chung1954stochastic}
Kai~Lai Chung.
\newblock On a stochastic approximation method.
\newblock \emph{The Annals of Mathematical Statistics}, pages 463--483, 1954.

\bibitem[Dean et~al.(2012)Dean, Corrado, Monga, Chen, Devin, Mao, Senior,
  Tucker, Yang, Le, et~al.]{dean2012large}
Jeffrey Dean, Greg Corrado, Rajat Monga, Kai Chen, Matthieu Devin, Mark Mao,
  Andrew Senior, Paul Tucker, Ke~Yang, Quoc~V Le, et~al.
\newblock Large scale distributed deep networks.
\newblock In \emph{Advances in neural information processing systems}, pages
  1223--1231, 2012.

\bibitem[Dekel et~al.(2012)Dekel, Gilad-Bachrach, Shamir, and
  Xiao]{dekel2012optimal}
Ofer Dekel, Ran Gilad-Bachrach, Ohad Shamir, and Lin Xiao.
\newblock Optimal distributed online prediction using mini-batches.
\newblock \emph{Journal of Machine Learning Research}, 13\penalty0
  (Jan):\penalty0 165--202, 2012.

\bibitem[Ghadimi and Lan(2013)]{ghadimi2013stochastic}
Saeed Ghadimi and Guanghui Lan.
\newblock Stochastic first-and zeroth-order methods for nonconvex stochastic
  programming.
\newblock \emph{SIAM Journal on Optimization}, 23\penalty0 (4):\penalty0
  2341--2368, 2013.

\bibitem[Hazan and Kale(2014)]{hazan2014beyond}
Elad Hazan and Satyen Kale.
\newblock Beyond the regret minimization barrier: optimal algorithms for
  stochastic strongly-convex optimization.
\newblock \emph{The Journal of Machine Learning Research}, 15\penalty0
  (1):\penalty0 2489--2512, 2014.

\bibitem[Johnson and Zhang(2013)]{johnson2013accelerating}
Rie Johnson and Tong Zhang.
\newblock Accelerating stochastic gradient descent using predictive variance
  reduction.
\newblock In \emph{Advances in neural information processing systems}, pages
  315--323, 2013.

\bibitem[Krizhevsky and Hinton(2009)]{krizhevsky2009learning}
Alex Krizhevsky and Geoffrey Hinton.
\newblock Learning multiple layers of features from tiny images.
\newblock 2009.

\bibitem[Lian et~al.(2015)Lian, Huang, Li, and Liu]{lian2015asynchronous}
Xiangru Lian, Yijun Huang, Yuncheng Li, and Ji~Liu.
\newblock Asynchronous parallel stochastic gradient for nonconvex optimization.
\newblock In \emph{Advances in Neural Information Processing Systems}, pages
  2737--2745, 2015.

\bibitem[Lin et~al.(2013)Lin, Chen, and Yan]{lin2013network}
Min Lin, Qiang Chen, and Shuicheng Yan.
\newblock Network in network.
\newblock \emph{arXiv preprint arXiv:1312.4400}, 2013.

\bibitem[Loshchilov and Hutter(2016)]{loshchilov2016sgdr}
Ilya Loshchilov and Frank Hutter.
\newblock Sgdr: stochastic gradient descent with restarts.
\newblock \emph{Learning}, 10:\penalty0 3, 2016.

\bibitem[Nemirovski et~al.(2009)Nemirovski, Juditsky, Lan, and
  Shapiro]{nemirovski2009robust}
Arkadi Nemirovski, Anatoli Juditsky, Guanghui Lan, and Alexander Shapiro.
\newblock Robust stochastic approximation approach to stochastic programming.
\newblock \emph{SIAM Journal on optimization}, 19\penalty0 (4):\penalty0
  1574--1609, 2009.

\bibitem[Recht et~al.(2011)Recht, Re, Wright, and Niu]{recht2011hogwild}
Benjamin Recht, Christopher Re, Stephen Wright, and Feng Niu.
\newblock Hogwild: A lock-free approach to parallelizing stochastic gradient
  descent.
\newblock In \emph{Advances in neural information processing systems}, pages
  693--701, 2011.

\bibitem[Robbins and Monro(1951)]{robbins1951stochastic}
Herbert Robbins and Sutton Monro.
\newblock A stochastic approximation method.
\newblock \emph{The annals of mathematical statistics}, pages 400--407, 1951.

\bibitem[Robbins and Siegmund(1971)]{robbins1971convergence}
Herbert Robbins and David Siegmund.
\newblock A convergence theorem for non negative almost supermartingales and
  some applications.
\newblock In \emph{Optimizing methods in statistics}, pages 233--257. Elsevier,
  1971.

\bibitem[Sacks(1958)]{sacks1958asymptotic}
Jerome Sacks.
\newblock Asymptotic distribution of stochastic approximation procedures.
\newblock \emph{The Annals of Mathematical Statistics}, 29\penalty0
  (2):\penalty0 373--405, 1958.

\bibitem[Shamir and Zhang(2013)]{shamir2013stochastic}
Ohad Shamir and Tong Zhang.
\newblock Stochastic gradient descent for non-smooth optimization: Convergence
  results and optimal averaging schemes.
\newblock In \emph{International Conference on Machine Learning}, pages 71--79,
  2013.

\bibitem[Simonyan and Zisserman(2014)]{simonyan2014very}
Karen Simonyan and Andrew Zisserman.
\newblock Very deep convolutional networks for large-scale image recognition.
\newblock \emph{arXiv preprint arXiv:1409.1556}, 2014.

\bibitem[Smith et~al.(2016)Smith, Forte, Ma, Tak{\'a}c, Jordan, and
  Jaggi]{smith2016cocoa}
Virginia Smith, Simone Forte, Chenxin Ma, Martin Tak{\'a}c, Michael~I Jordan,
  and Martin Jaggi.
\newblock Cocoa: A general framework for communication-efficient distributed
  optimization.
\newblock \emph{arXiv preprint arXiv:1611.02189}, 2016.

\bibitem[Wang et~al.(2017)Wang, Wang, and Srebro]{wang2017memory}
Jialei Wang, Weiran Wang, and Nathan Srebro.
\newblock Memory and communication efficient distributed stochastic
  optimization with minibatch prox.
\newblock \emph{arXiv preprint arXiv:1702.06269}, 2017.

\bibitem[Zhang et~al.(2016)Zhang, De~Sa, Mitliagkas, and
  R{\'e}]{zhang2016parallel}
Jian Zhang, Christopher De~Sa, Ioannis Mitliagkas, and Christopher R{\'e}.
\newblock Parallel sgd: When does averaging help?
\newblock \emph{arXiv preprint arXiv:1606.07365}, 2016.

\bibitem[Zhang et~al.(2015)Zhang, Choromanska, and LeCun]{zhang2015deep}
Sixin Zhang, Anna~E Choromanska, and Yann LeCun.
\newblock Deep learning with elastic averaging sgd.
\newblock In \emph{Advances in Neural Information Processing Systems}, pages
  685--693, 2015.

\bibitem[Zinkevich et~al.(2010)Zinkevich, Weimer, Li, and
  Smola]{zinkevich2010parallelized}
Martin Zinkevich, Markus Weimer, Lihong Li, and Alex~J Smola.
\newblock Parallelized stochastic gradient descent.
\newblock In \emph{Advances in neural information processing systems}, pages
  2595--2603, 2010.

\end{thebibliography}

\end{document}